\documentclass[10pt,twocolumn,letterpaper]{article}

\usepackage{cvpr}
\usepackage{times}
\usepackage{epsfig}
\usepackage{graphicx}
\usepackage{amsmath}
\usepackage{amssymb}

\usepackage{cases}
\usepackage[linesnumbered,ruled,vlined,lined,boxed,commentsnumbered]{algorithm2e}
\usepackage{subfigure}
\usepackage{indentfirst}
\usepackage{multirow}
\usepackage{boldline}

\usepackage{enumitem}
\setitemize{noitemsep,topsep=0pt,parsep=0pt,partopsep=0pt}

\usepackage[table]{xcolor}

\usepackage{amsthm}
\theoremstyle{definition}
\newtheorem{definition}{Definition}

\newtheorem{proposition}{Proposition}
\newtheorem{implication}{Implication}

\usepackage[pagebackref=true,breaklinks=true,letterpaper=true,colorlinks,bookmarks=false]{hyperref}

\cvprfinalcopy 

\def\cvprPaperID{4223} 

\ifcvprfinal\fi
\begin{document}

\title{Class-Balanced Loss Based on Effective Number of Samples}


\author{Yin Cui$^{1,2}$\thanks{The work was performed while Yin Cui and Yang Song worked at Google (a subsidiary of Alphabet Inc.).}\hspace{20pt}Menglin Jia$^{1}$\hspace{20pt}Tsung-Yi Lin$^{3}$\hspace{20pt}Yang Song$^{4}$\hspace{20pt}Serge Belongie$^{1,2}$\\
$^{1}$Cornell University\hspace{20pt}$^{2}$Cornell Tech\hspace{20pt}$^{3}$Google Brain\hspace{20pt}$^{4}$Alphabet Inc.}

\maketitle

\begin{abstract}
With the rapid increase of large-scale, real-world datasets, it becomes critical to address the problem of long-tailed data distribution (\ie, a few classes account for most of the data, while most classes are under-represented). 
Existing solutions typically adopt class re-balancing strategies such as re-sampling and re-weighting based on the number of observations for each class.
In this work, we argue that as the number of samples increases, the additional benefit of a newly added data point will diminish.
We introduce a novel theoretical framework to measure data overlap by associating with each sample a small neighboring region rather than a single point.
The effective number of samples is defined as the volume of samples and can be calculated by a simple formula $(1-\beta^{n})/(1-\beta)$, where $n$ is the number of samples and $\beta \in [0,1)$ is a hyperparameter.
We design a re-weighting scheme that uses the effective number of samples for each class to re-balance the loss, thereby yielding a \emph{class-balanced loss}. Comprehensive experiments are conducted on artificially induced long-tailed CIFAR datasets and large-scale datasets including ImageNet and iNaturalist.
Our results show that when trained with the proposed class-balanced loss, the network is able to achieve significant performance gains on long-tailed datasets. 
\end{abstract}

\begin{figure}[t]
\begin{center}
\includegraphics[width=0.95\columnwidth]{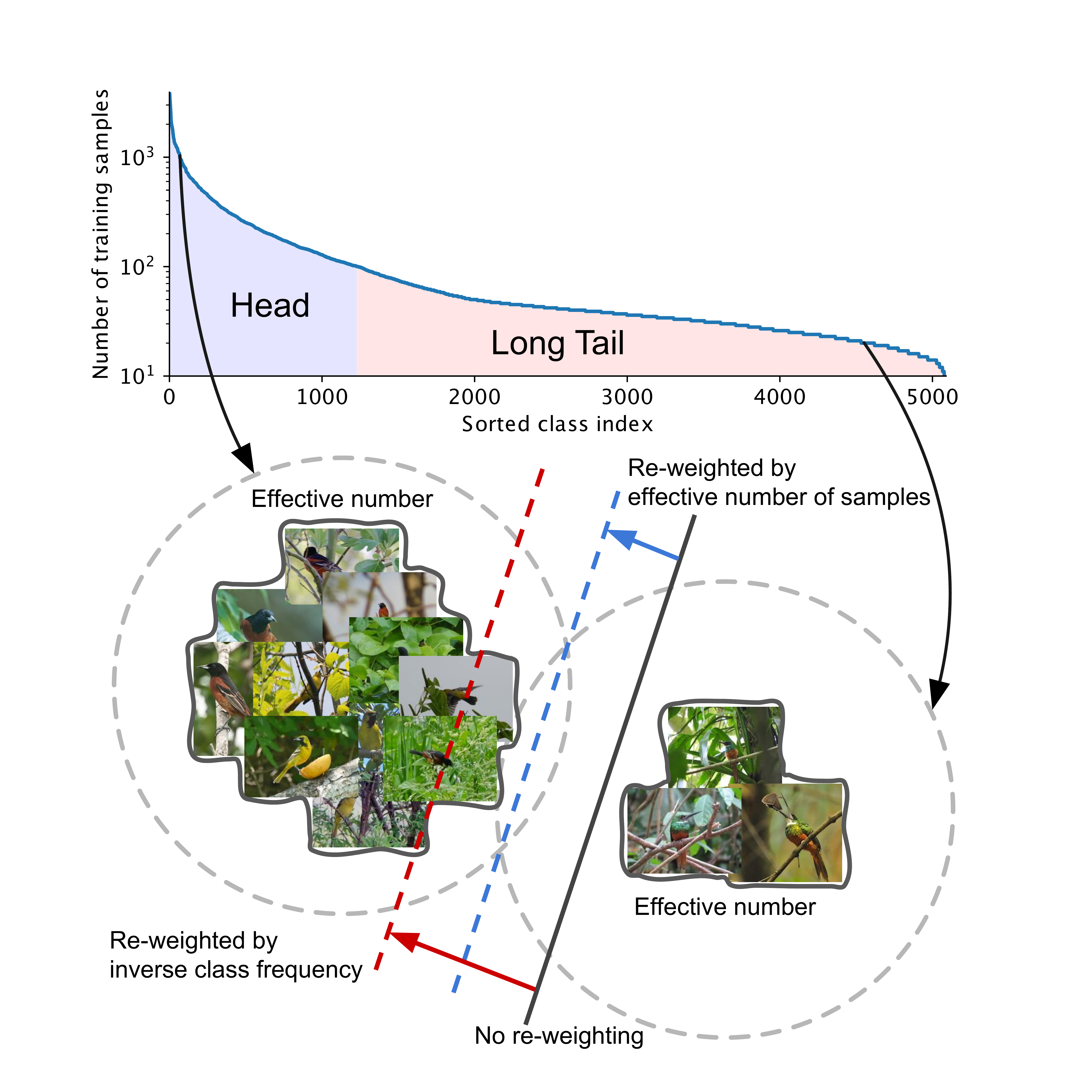}
\end{center}
\caption{Two classes, one from the head and one from the tail of a long-tailed dataset (iNaturalist 2017 \cite{inaturalist} in this example), have drastically different number of samples. Models trained on these samples are biased toward dominant classes (black solid line). Re-weighing the loss by inverse class frequency usually yields poor performance (red dashed line) on real-world data with high class imbalance. We propose a theoretical framework to quantify the effective number of samples by taking data overlap into consideration. A class-balanced term is designed to re-weight the loss by inverse effective number of samples. We show in experiments that the performance of a model can be improved when trained with the proposed class-balanced loss (blue dashed line).}
\label{fig:overview}
\end{figure}

\section{Introduction}
The recent success of deep Convolutional Neural Networks (CNNs) for visual recognition~\cite{alexnet, vggnet, googlenet, resnet} owes much to the availability of large-scale, real-world annotated datasets~\cite{imagenet, coco, places, inaturalist}.
In contrast with commonly used visual recognition datasets (\eg, CIFAR~\cite{cifar, tiny-images}, ImageNet ILSVRC 2012~\cite{imagenet, ilsvrc} and CUB-200 Birds~\cite{cub200}) that exhibit roughly uniform distributions of class labels, real-world datasets have skewed~\cite{kendall1946advanced} distributions, with a \emph{long-tail}: a few dominant classes claim most of the examples, while most of the other classes are represented by relatively few examples.
CNNs trained on such data perform poorly for weakly represented classes~\cite{japkowicz2002class, he2008learning, van2017devil, buda2018systematic}.

A number of recent studies have aimed to alleviate the challenge of long-tailed training data~\cite{bengio2015sharing, ouyang2016factors, huang2016learning, van2017devil, wang2017learning, geifman2017deep, zhang2017range, yin2018feature}.
In general, there are two strategies: re-sampling and cost-sensitive re-weighting.
In re-sampling, the number of examples is directly adjusted by over-sampling (adding repetitive data) for the minor class or under-sampling (removing data) for the major class, or both.
In cost-sensitive re-weighting, we influence the loss function by assigning relatively higher costs to examples from minor classes.
In the context of deep feature representation learning using CNNs, re-sampling may either introduce large amounts of duplicated samples, which slows down the training and makes the model susceptible to overfitting when over-sampling, or discard valuable examples that are important for feature learning when under-sampling.
Due to these disadvantages of applying re-sampling for CNN training, the present work focuses on re-weighting approaches, namely, how to design a better class-balanced loss.

Typically, a class-balanced loss assigns sample weights inversely proportionally to the class frequency.
This simple heuristic method has been widely adopted~\cite{huang2016learning,wang2017learning}.
However, recent work on training from large-scale, real-world, long-tailed datasets~\cite{word2vec, uru} reveals poor performance when using this strategy.
Instead, they use a ``smoothed'' version of weights that are empirically set to be inversely proportional to the square root of class frequency.
These observations suggest an interesting question: how can we design a better class-balanced loss that is applicable to a diverse array of datasets?

We aim to answer this question from the perspective of sample size.
As illustrated in Figure~\ref{fig:overview}, we consider training a model
to discriminate between a major class and a minor class from a long-tailed dataset.
Due to highly imbalanced data, directly training the model or re-weighting the loss by inverse number of samples cannot yield satisfactory performance.
Intuitively, the more data, the better.
However, since there is information overlap among data, as the number of samples increases, the marginal benefit a model can extract from the data diminishes.
In light of this, we propose a novel theoretical framework to characterize data overlap and calculate the effective number of samples in a model- and loss-agnostic manner.
A class-balanced re-weighting term that is inversely proportional to the effective number of samples is added to the loss function.
Extensive experimental results indicate that this class-balanced term provides a significant boost to the performance of commonly used loss functions for training CNNs on long-tailed datasets.

Our key contributions can be summarized as follows:
(1) We provide a theoretical framework to study the effective number of samples and show how to design a class-balanced term to deal with long-tailed training data.
(2) We show that significant performance improvements can be achieved by adding the proposed class-balanced term to existing commonly used loss functions including softmax cross-entropy, sigmoid cross-entropy and focal loss.
In addition, we show our class-balanced loss can be used as a generic loss for visual recognition by outperforming commonly-used softmax cross-entropy loss on ILSVRC 2012.
We believe our study on quantifying the effective number of samples and class-balanced loss can offer useful guidelines for researchers working in domains with long-tailed class distributions.

\section{Related Work}
Most of previous efforts on long-tailed imbalanced data can be divided into two regimes: re-sampling~\cite{shen2016relay,geifman2017deep, buda2018systematic, zou2018unsupervised} (including over-sampling and under-sampling) and cost-sensitive learning~\cite{ting2000comparative, zhou2006training, huang2016learning, khan2018cost, sarafianos2018deep}.

\textbf{Re-Sampling.} Over-sampling adds repeated samples from minor classes, which could cause the model to overfit.
To solve this, novel samples can be either interpolated from neighboring samples~\cite{chawla2002smote} or synthesized~\cite{he2008adasyn, zou2018unsupervised} for minor classes.
However, the model is still error-prone due to noise in the novel samples.
It was argued that even if over-sampling incurs risks from removing important samples, under-sampling is still preferred over over-sampling~\cite{drummond2003c4}.

\textbf{Cost-Sensitive Learning.} Cost-Sensitive Learning can be traced back to a classical method in statistics called importance sampling~\cite{kahn1953methods}, where weights are assigned to samples in order to match a given data distribution.
Elkan~\etal~\cite{elkan2001foundations} studied how to assign weights to adjust the decision boundary to match a given target in the case of binary classification.
For imbalanced datasets, weighting by inverse class frequency~\cite{huang2016learning, wang2017learning} or a smoothed version of inverse square root of class frequency~\cite{word2vec, uru} are often adopted.
As a generalization of smoothed weighting with a theoretically grounded framework, we focus on (a) how to quantify the effective number of samples and (b) using it to re-weight the loss.
Another line of important work aims to study sample difficulty in terms of loss and assign higher weights to hard examples~\cite{boosting, malisiewicz2011ensemble, dong2017class, focal_loss}.
Samples from minor classes tend to have higher losses than those from major classes as the features learned in minor classes are usually poorer.
However, there is no direct connection between sample difficulty and the number of samples.
A side effect of assigning higher weights to hard examples is the focus on harmful samples (\eg, noisy data or mislabeled data)~\cite{koh2017understanding, ren2018learning}.
In our work, we do not make any assumptions on the sample difficulty and data distribution.
By improving the focal loss~\cite{focal_loss} using our class-balanced term in experiments, we show that our method is complementary to re-weighting based on sample difficulty.

It is noteworthy to mention that previous work has also explored other ways of dealing with data imbalance, including transferring the knowledge learned from major classes to minor classes~\cite{bengio2015sharing, ouyang2016factors, wang2017learning, cui2018large, yin2018feature} and designing a better training objective via metric learning~\cite{huang2016learning, zhang2017range, you2018scalable}.

\textbf{Covering and Effective Sample Size.}
Our theoretical framework is inspired by the random covering problem~\cite{janson1986random}, where the goal is to cover a large set by a sequence of i.i.d. random small sets.
We simplify the problem in Section~\ref{sec:effec_num} by making reasonable assumptions.
Note that the effective number of samples proposed in this paper is different from the concept of effective sample size in statistics.
The effective sample size is used to calculate variance when samples are correlated.

\section{Effective Number of Samples}
\label{sec:effec_num}
We formulate the data sampling process as a simplified version of random covering.
The key idea is to associate each sample with a small neighboring region instead of a single point.
We present our theoretical framework and the formulation of calculating effective number of samples.

\subsection{Data Sampling as Random Covering}
Given a class, denote the set of all possible data in the feature space of this class as $\mathcal{S}$.
We assume the volume of $\mathcal{S}$ is $N$ and $N \geq 1$.
Denote each data as a subset of $\mathcal{S}$ that has the unit volume of $1$ and may overlap with other data.
Consider the data sampling process as a random covering problem where each data (subset) is randomly sampled from $\mathcal{S}$ to cover the entire set of $\mathcal{S}$.
The more data is being sampled, the better the coverage of $\mathcal{S}$ is.
The expected total volume of sampled data increases as the number of data increases and is bounded by $N$. Therefore, we define:

\theoremstyle{definition}
\begin{definition}[Effective Number]
\label{def:1}
The \textit{effective number} of samples is the expected volume of samples.
\end{definition}

The calculation of the expected volume of samples is a very difficult problem that depends on the shape of the sample and the dimensionality of the feature space~\cite{janson1986random}.
To make the problem tamable, we simplify the problem by not considering the situation of partial overlapping.
That is, we assume a newly sampled data can only interact with previously sampled data in two ways: either entirely inside the set of previously sampled data with the probability of $p$ or entirely outside with the probability of $1-p$, as illustrated in Figure~\ref{fig:covering}.
As the number of sampled data points increases, the probability $p$ will be higher.

Before we dive into the mathematical formulations, we discuss the connection between our definition of effective number of samples and real-world visual data.
Our idea is to capture the diminishing marginal benefits by using more data points of a class.
Due to intrinsic similarities among real-world data, as the number of samples grows, it is highly possible that a newly added sample is a near-duplicate of existing samples.
In addition, CNNs are trained with heavy data augmentations, where simple transformations such as random cropping, re-scaling and horizontal flipping will be applied to the input data.
In this case, all augmented examples are also considered as same with the original example.
Presumably, the stronger the data augmentation is, the smaller the $N$ will be.
The small neighboring region of a sample is a way to capture all near-duplicates and instances that can be obtained by data augmentation.
For a class, $N$ can be viewed as the number of \textit{unique prototypes}.

\begin{figure}[t]
\begin{center}
\includegraphics[width=0.95\columnwidth]{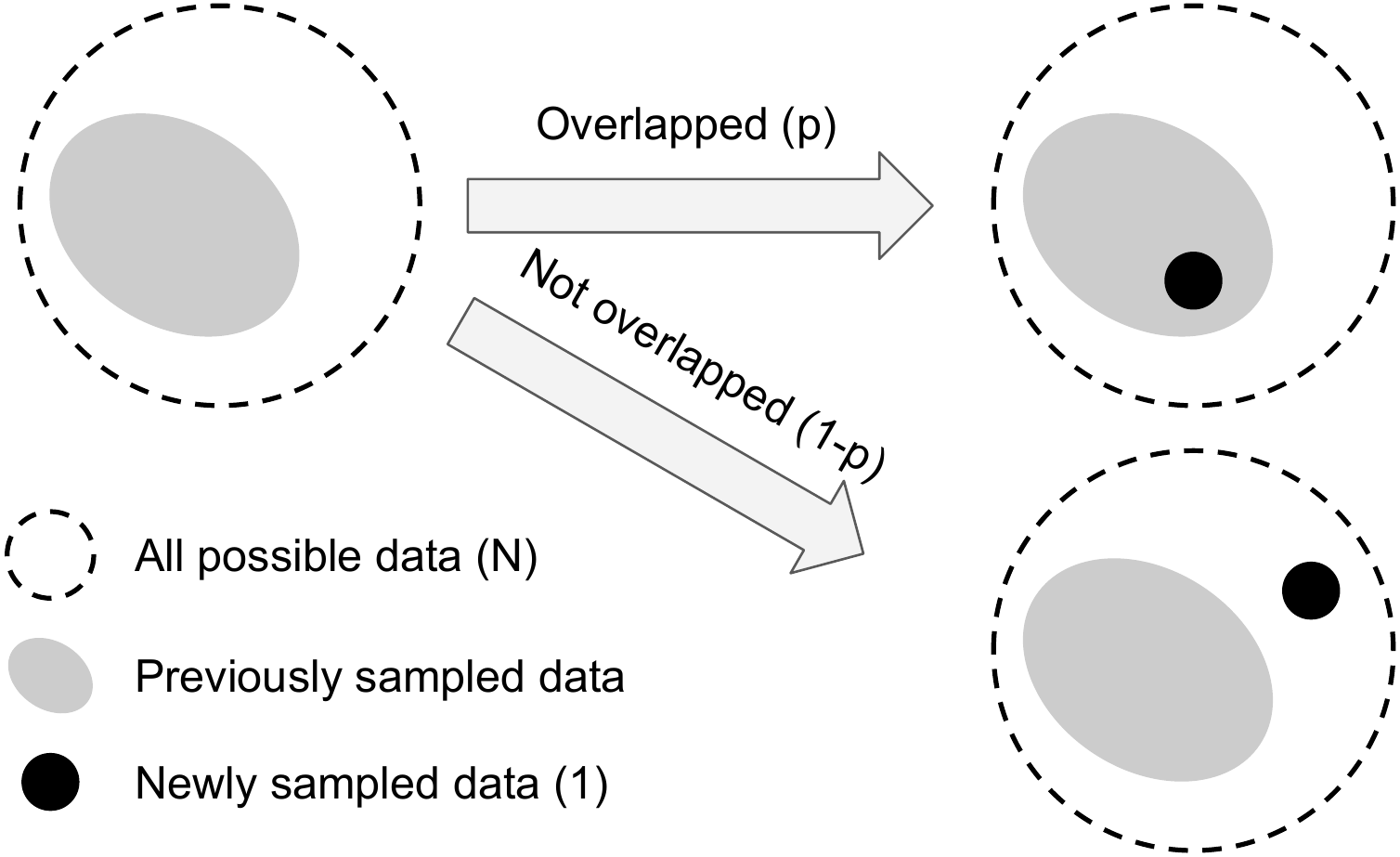}
\end{center}
\caption{Giving the set of all possible data with volume $N$ and the set of previously sampled data, a new sample with volume $1$ has the probability of $p$ being overlapped with previous data and the probability of $1-p$ not being overlapped.}
\label{fig:covering}
\end{figure}

\subsection{Mathematical Formulation}
Denote the effective number (expected volume) of samples as $E_n$, where $n \in \mathbb{Z}_{>0}$ is the number of samples. 

\begin{proposition}[Effective Number]
\label{prop:1}
$E_n = (1-\beta^n)/(1-\beta)$, where $\beta = (N-1)/N$.
\end{proposition}

\begin{proof}
We prove the proposition by induction.
It is obvious that $E_1 = 1$ because there is no overlapping.
So $E_1 = (1-\beta^1)/(1-\beta) = 1$ holds.
Now let's consider a general case where we have previously sampled $n-1$ examples and are about to sample the $n^{th}$ example.
Now the expected volume of previously sampled data is $E_{n-1}$ and the newly sampled data point has the probability of $p = E_{n-1} / N$ to be overlapped with previous samples. Therefore, the expected volume after sampling $n^{th}$ example is:
\begin{equation}
\label{eqn:1}
E_n = p E_{n-1} + (1-p) (E_{n-1} + 1) = 1 + \frac{N-1}{N} E_{n-1}.
\end{equation}
Assume $E_{n-1} = (1-\beta^{n-1})/(1-\beta)$ holds, then
\begin{equation}
\label{eqn:2}
E_n = 1 + \beta \frac{1-\beta^{n-1}}{1-\beta} = \frac{1-\beta+\beta-\beta^{n}}{1-\beta} = \frac{1-\beta^{n}}{1-\beta}.
\end{equation}
\end{proof}

The above proposition shows that the effective number of samples is an exponential function of $n$.
The hyperparameter $\beta \in [0, 1)$ controls how fast $E_n$ grows as $n$ increases.

Another explanation of the effective number $E_n$ is:
\begin{equation}
\label{eqn:3}
E_n = (1-\beta^n)/(1-\beta) = \sum_{j=1}^{n} \beta^{j-1}.
\end{equation}
This means that the $j^{th}$ sample contributes $\beta^{j-1}$ to the effective number. The total volume $N$ for all possible data in the class can then be calculated as:
\begin{equation}
N = \lim_{n \to \infty}\sum_{j=1}^{n} \beta^{j-1} = 1 / (1 - \beta).
\end{equation}
This is consistent with our definition of $\beta$ in the proposition.

\begin{implication}[Asymptotic Properties]
\label{impl:1}
$E_n = 1$ if $\beta=0$ ($N=1$). $E_n \to n$ as $\beta \to 1$ ($N \to \infty$).
\end{implication}

\begin{proof}
If $\beta=0$, then $E_n = (1 - 0^n)/(1-0) = 1$.
In the case of $\beta \to 1$, denote $f(\beta) = 1-\beta^{n}$ and $g(\beta) = 1-\beta$.
Since $\lim_{\beta \to 1}f(\beta) = \lim_{\beta \to 1}g(\beta) = 0$, $g'(\beta) = -1 \neq 0$ and $\lim_{\beta \to 1} f'(\beta) / g'(\beta) = \lim_{\beta \to 1} (-n \beta^{n-1})/(-1) = n$ exists, using L'H\^{o}pital's rule, we have
\begin{equation}
\label{eqn:4}
\lim_{\beta \to 1} E_n = \lim_{\beta \to 1} \frac{f(\beta)}{g(\beta)} = \lim_{\beta \to 1} \frac{f'(\beta)}{g'(\beta)} = n. 
\end{equation}
\end{proof}

The asymptotic property of $E_n$ shows that when $N$ is large, the effective number of samples is same as the number of samples $n$.
In this scenario, we think the number of unique prototypes $N$ is large, thus there is no data overlap and every sample is unique.
On the other extreme, if $N=1$, this means that we believe there exist a single prototype so that all the data in this class can be represented by this prototype via data augmentation, transformations,~\etc.

\section{Class-Balanced Loss}
\label{sec:cb_loss}
The \textit{Class-Balanced Loss} is designed to address the problem of training from imbalanced data by introducing a weighting factor that is inversely proportional to the effective number of samples.
The class-balanced loss term can be applied to a wide range of deep networks and loss functions.

For an input sample $\mathbf{x}$ with label $y \in \{1, 2, \dots, C\}$~\footnote{For simplicity, we derive the loss function by assuming there is only one ground-truth label for a sample.}, where $C$ is the total number of classes, suppose the model's estimated class probabilities are $\mathbf{p} = [p_1, p_2, \dots, p_C]^\top$, where $p_i \in [0, 1]\ \forall\ i$, we denote the loss as $\mathcal{L}(\mathbf{p}, y)$.
Suppose the number of samples for class $i$ is $n_i$, based on Equation~\ref{eqn:2}, the proposed effective number of samples for class $i$ is $E_{n_i} = (1 - \beta_i^{n_i})/(1 - \beta_i)$, where $\beta_i = (N_i-1)/N_i$.
Without further information of data for each class, it is difficult to empirically find a set of good hyperparameters $N_i$ for all classes.
Therefore, in practice, we assume $N_i$ is only dataset-dependent and set $N_i = N$, $\beta_i = \beta = (N-1)/N$ for all classes in a dataset.

To balance the loss, we introduce a weighting factor $\alpha_i$ that is inversely proportional to the effective number of samples for class $i$: $\alpha_i \propto 1/E_{n_i}$.
To make the total loss roughly in the same scale when applying $\alpha_i$, we normalize $\alpha_i$ so that $\sum_{i=1}^C \alpha_i = C$.
For simplicity, we abuse the notation of $1/E_{n_i}$ to denote the normalized weighting factor in the rest of our paper.


Formally speaking, given a sample from class $i$ that contains $n_i$ samples in total, we propose to add a weighting factor $(1 - \beta)/(1 - \beta^{n_i})$ to the loss function, with hyperparameter $\beta \in [0, 1)$. The class-balanced (CB) loss can be written as:
\begin{equation}
\label{eqn:cb_loss}
\text{CB}(\mathbf{p}, y) = \frac{1}{E_{n_y}} \mathcal{L}(\mathbf{p}, y) = \frac{1 - \beta}{1 - \beta^{n_y}}\ \mathcal{L}(\mathbf{p}, y),
\end{equation}
where $n_y$ is the number of samples in the ground-truth class $y$.
We visualize class-balanced loss in Figure~\ref{fig:cb_loss} as a function of $n_y$ for different $\beta$.
Note that $\beta=0$ corresponds to no re-weighting and $\beta \to 1$ corresponds to re-weighing by inverse class frequency.
The proposed novel concept of effective number of samples enables us to use a hyperparameter $\beta$ to smoothly adjust the class-balanced term between no re-weighting and re-weighing by inverse class frequency.

The proposed class-balanced term is model-agnostic and loss-agnostic in the sense that it's independent to the choice of loss function $\mathcal{L}$ and predicted class probabilities $\mathbf{p}$.
To demonstrate the proposed class-balanced loss is generic, we show how to apply class-balanced term to three commonly used loss functions: softmax cross-entropy loss, sigmoid cross-entropy loss and focal loss.

\begin{figure}[t]
\begin{center}
\includegraphics[width=0.95\columnwidth]{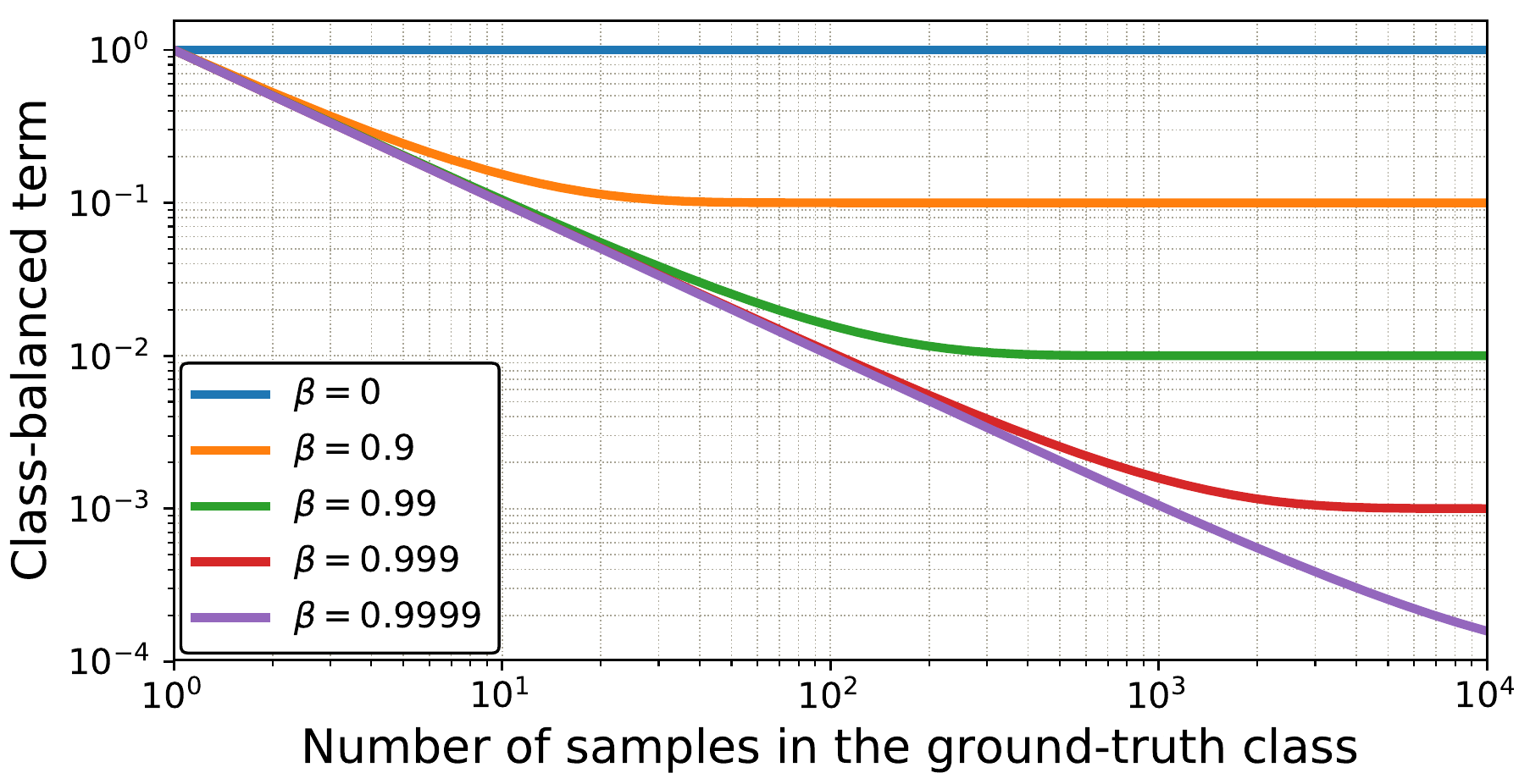}
\end{center}
\caption{Visualization of the proposed class-balanced term $(1 - \beta)/(1 - \beta^{n_y})$, where $n_y$ is the number of samples in the ground-truth class. Both axes are in log scale. For a long-tailed dataset where major classes have significantly more samples than minor classes, setting $\beta$ properly re-balances the relative loss across classes and reduces the drastic imbalance of re-weighing by inverse class frequency.}
\label{fig:cb_loss}
\end{figure}

\subsection{Class-Balanced Softmax Cross-Entropy Loss}
Suppose the predicted output from the model for all classes are $\mathbf{z} = [z_1, z_2, \dots, z_C]^\top$, where $C$ is the total number of classes.
The softmax function regards each class as mutual exclusive and calculate the probability distribution over all classes as $p_i = \exp(z_i)/\sum_{j=1}^C \exp(z_j), \forall\ i \in \{1, 2, \dots, C\}$.
Given a sample with class label $y$, the softmax cross-entropy (CE) loss for this sample is written as:
\begin{equation}
\text{CE}_{\text{softmax}}(\mathbf{z}, y) = -\log\left(\frac{\exp(z_y)}{\sum_{j=1}^C \exp(z_j)}\right).
\end{equation}
Suppose class $y$ has $n_y$ training samples, the class-balanced (CB) softmax cross-entropy loss is:
\begin{equation}
\text{CB}_{\text{softmax}}(\mathbf{z}, y) = -\frac{1 - \beta}{1 - \beta^{n_y}} \log\left(\frac{\exp(z_y)}{\sum_{j=1}^C \exp(z_j)}\right).
\end{equation}

\subsection{Class-Balanced Sigmoid Cross-Entropy Loss}
Different from softmax, class-probabilities calculated by sigmoid function assume each class is independent and not mutually exclusive.
When using sigmoid function, we regard multi-class visual recognition as multiple binary classification tasks, where each output node of the network is performing a one-vs-all classification to predict the probability of the target class over the rest of classes.
Compared with softmax, sigmoid presumably has two advantages for real-world datasets:
(1) Sigmoid doesn't assume the mutual exclusiveness among classes, which aligns well with real-world data, where a few classes might be very similar to each other, especially in the case of large number of fine-grained classes.
(2) Since each class is considered independent and has its own predictor, sigmoid unifies single-label classification with multi-label prediction. This is a nice property to have since real-world data often has more than one semantic label.

Using same notations as softmax cross-entropy, for simplicity, we define $z_i^t$ as:
\begin{equation}
z_i^t = \begin{cases}z_i,       & \text{if}\ i = y.\\
-z_i,  & \text{otherwise}.
\end{cases}
\end{equation}
Then the sigmoid cross-entropy (CE) loss can be written as:
\begin{equation}
\begin{aligned}
\text{CE}_{\text{sigmoid}}(\mathbf{z}, y) &= - \sum_{i=1}^C \log\left(\text{sigmoid}(z_i^t)\right) \\
&= - \sum_{i=1}^C \log\left(\frac{1}{1 + \exp(-z_i^t)}\right).
\end{aligned}
\end{equation}
The class-balanced (CB) sigmoid cross-entropy loss is:
\begin{equation}
\text{CB}_{\text{sigmoid}}(\mathbf{z}, y) = - \frac{1 - \beta}{1 - \beta^{n_y}} \sum_{i=1}^C \log\left(\frac{1}{1 + \exp(-z_i^t)}\right).
\end{equation}

\subsection{Class-Balanced Focal Loss}
The recently proposed focal loss (FL)~\cite{focal_loss} adds a modulating factor to the sigmoid cross-entropy loss to reduce the relative loss for well-classified samples and focus on difficult samples.
Denote $p_i^t = \text{sigmoid}(z_i^t) = 1/(1 + \exp(-z_i^t))$, the focal loss can be written as:
\begin{equation}
\text{FL}(\mathbf{z}, y) = - \sum_{i=1}^C (1 - p_i^t)^{\gamma}\ \log(p_i^t).
\end{equation}
The class-balanced (CB) focal loss is:
\begin{equation}
\text{CB}_{\text{focal}}(\mathbf{z}, y) = - \frac{1 - \beta}{1 - \beta^{n_y}} \sum_{i=1}^C (1 - p_i^t)^{\gamma}\ \log(p_i^t).
\end{equation}

The original focal loss has an $\alpha$-balanced variant.
The class-balanced focal loss is same as $\alpha$-balanced focal loss when $\alpha_t = (1 - \beta)/(1 - \beta^{n_y})$.
Therefore, the class-balanced term can be viewed as an explicit way to set $\alpha_t$ in focal loss based on the effective number of samples.

\section{Experiments}
The proposed class-balanced losses are evaluated on artificially created long-tailed CIFAR~\cite{cifar} datasets with controllable degrees of data imbalance and real-world long-tailed datasets iNaturalist 2017~\cite{inaturalist} and 2018~\cite{inat18}.
To demonstrate our loss is generic for visual recognition, we also present experiments on ImageNet data (ILSVRC 2012~\cite{ilsvrc}).
We use deep residual networks (ResNet)~\cite{resnet} with various depths and train all networks from scratch.

\begin{figure}[t]
\begin{center}
\includegraphics[width=0.95\columnwidth]{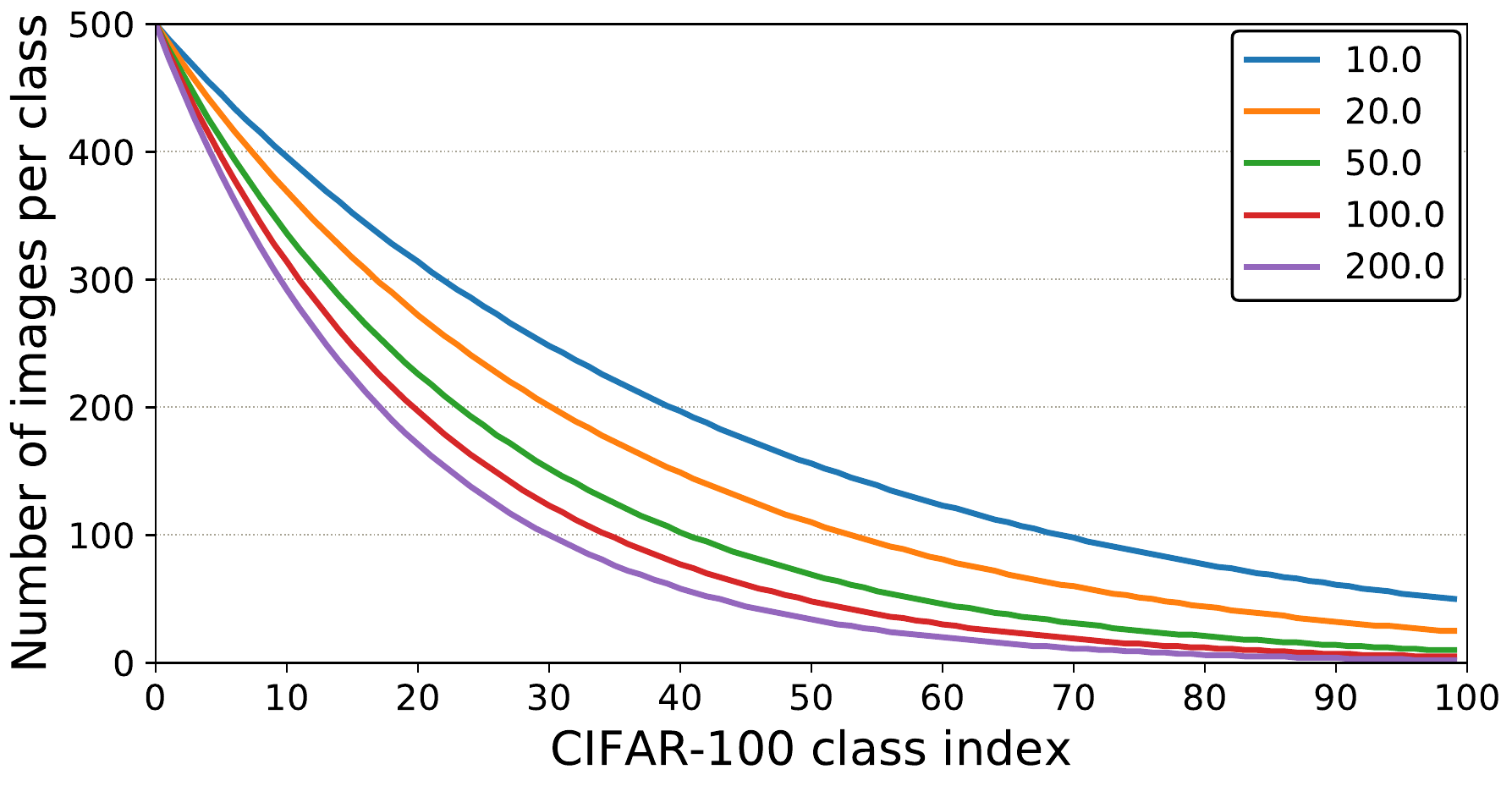}
\end{center}
\caption{Number of training samples per class in artificially created long-tailed CIFAR-100 datasets with different imbalance factors.}
\label{fig:cifar_data}
\end{figure}

\begin{table}[t]
\begin{center}
\begin{tabular}{ |l|r|r| } 
\hline
Dataset Name & \# Classes & Imbalance \\ \hline 
Long-Tailed CIFAR-10    & 10 & 10.00 - 200.00 \\
Long-Tailed CIFAR-100   & 100 & 10.00 - 200.00 \\
iNaturalist 2017        & 5,089 & 435.44 \\
iNaturalist 2018        & 8,142 & 500.00 \\
ILSVRC 2012    & 1,000 & 1.78 \\
\hline
\end{tabular}
\end{center}
\caption{Datasets that are used to evaluate the effectiveness of class-balanced loss. We created 5 long-tailed versions of both CIFAR-10 and CIFAR-100 with imbalance factors of 10, 20, 50, 100 and 200 respectively.}
\label{tab:dataset}
\end{table}

\begin{table*}[t]
\small
\begin{center}
\begin{tabular}{ lcccccccccccc }
\hline
\multicolumn{1}{l}{Dataset Name} &  \multicolumn{6}{|c}{Long-Tailed CIFAR-10} & \multicolumn{6}{|c}{Long-Tailed CIFAR-100} \\
\hline
 \multicolumn{1}{l}{Imbalance} & \multicolumn{1}{|c}{200} & \multicolumn{1}{|c}{100} & \multicolumn{1}{|c}{50} & \multicolumn{1}{|c}{20} & \multicolumn{1}{|c}{10} & \multicolumn{1}{|c}{1} & \multicolumn{1}{|c}{200} & \multicolumn{1}{|c}{100} & \multicolumn{1}{|c}{50} & \multicolumn{1}{|c}{20} & \multicolumn{1}{|c}{10} & \multicolumn{1}{|c}{1} \\ \hline
\multicolumn{1}{l}{Softmax} & \multicolumn{1}{|r}{\textbf{34.32}} & \multicolumn{1}{|r}{29.64} & \multicolumn{1}{|r}{25.19} & \multicolumn{1}{|r}{17.77} & \multicolumn{1}{|r}{13.61} & \multicolumn{1}{|r}{6.61} & \multicolumn{1}{|r}{65.16} & \multicolumn{1}{|r}{61.68} & \multicolumn{1}{|r}{56.15} & \multicolumn{1}{|r}{48.86} & \multicolumn{1}{|r}{44.29} & \multicolumn{1}{|r}{29.07} \\
\multicolumn{1}{l}{Sigmoid} & \multicolumn{1}{|r}{34.51} & \multicolumn{1}{|r}{\textbf{29.55}} & \multicolumn{1}{|r}{23.84} & \multicolumn{1}{|r}{\textbf{16.40}} & \multicolumn{1}{|r}{\textbf{12.97}} & \multicolumn{1}{|r}{\textbf{6.36}} & \multicolumn{1}{|r}{64.39} & \multicolumn{1}{|r}{\textbf{61.22}} & \multicolumn{1}{|r}{55.85} & \multicolumn{1}{|r}{48.57} & \multicolumn{1}{|r}{44.73} & \multicolumn{1}{|r}{\textbf{28.39}} \\
\multicolumn{1}{l}{Focal ($\gamma = 0.5$)} & \multicolumn{1}{|r}{36.00} & \multicolumn{1}{|r}{29.77} & \multicolumn{1}{|r}{\textbf{23.28}} & \multicolumn{1}{|r}{17.11} & \multicolumn{1}{|r}{13.19} & \multicolumn{1}{|r}{6.75} & \multicolumn{1}{|r}{65.00} & \multicolumn{1}{|r}{61.31} & \multicolumn{1}{|r}{55.88} & \multicolumn{1}{|r}{48.90} & \multicolumn{1}{|r}{44.30} & \multicolumn{1}{|r}{28.55} \\
\multicolumn{1}{l}{Focal ($\gamma = 1.0$)} & \multicolumn{1}{|r}{34.71} & \multicolumn{1}{|r}{29.62} & \multicolumn{1}{|r}{23.29} & \multicolumn{1}{|r}{17.24} & \multicolumn{1}{|r}{13.34} & \multicolumn{1}{|r}{6.60} & \multicolumn{1}{|r}{\textbf{64.38}} & \multicolumn{1}{|r}{61.59} & \multicolumn{1}{|r}{\textbf{55.68}} & \multicolumn{1}{|r}{\textbf{48.05}} & \multicolumn{1}{|r}{\textbf{44.22}} & \multicolumn{1}{|r}{28.85} \\
\multicolumn{1}{l}{Focal ($\gamma = 2.0$)} & \multicolumn{1}{|r}{35.12} & \multicolumn{1}{|r}{30.41} & \multicolumn{1}{|r}{23.48} & \multicolumn{1}{|r}{16.77} & \multicolumn{1}{|r}{13.68} & \multicolumn{1}{|r}{6.61} & \multicolumn{1}{|r}{65.25} & \multicolumn{1}{|r}{61.61} & \multicolumn{1}{|r}{56.30} & \multicolumn{1}{|r}{48.98} & \multicolumn{1}{|r}{45.00} & \multicolumn{1}{|r}{28.52} \\ \hline \hline
\multicolumn{1}{l}{Class-Balanced} & \multicolumn{1}{|r}{\textbf{31.11}} & \multicolumn{1}{|r}{\textbf{25.43}} & \multicolumn{1}{|r}{\textbf{20.73}} & \multicolumn{1}{|r}{\textbf{15.64}} & \multicolumn{1}{|r}{\textbf{12.51}} & \multicolumn{1}{|r}{\textbf{6.36}$^*$} & \multicolumn{1}{|r}{\textbf{63.77}} & \multicolumn{1}{|r}{\textbf{60.40}} & \multicolumn{1}{|r}{\textbf{54.68}} & \multicolumn{1}{|r}{\textbf{47.41}} & \multicolumn{1}{|r}{\textbf{42.01}} & \multicolumn{1}{|r}{\textbf{28.39}$^*$} \\ \hline
\multicolumn{1}{l}{Loss Type} & \multicolumn{1}{|r}{SM} & \multicolumn{1}{|r}{Focal} & \multicolumn{1}{|r}{Focal} & \multicolumn{1}{|r}{SM} & \multicolumn{1}{|r}{SGM} & \multicolumn{1}{|r}{SGM} & \multicolumn{1}{|r}{Focal} & \multicolumn{1}{|r}{Focal} & \multicolumn{1}{|r}{SGM} & \multicolumn{1}{|r}{Focal} & \multicolumn{1}{|r}{Focal} & \multicolumn{1}{|r}{SGM} \\
\multicolumn{1}{l}{$\beta$} & \multicolumn{1}{|r}{0.9999} & \multicolumn{1}{|r}{0.9999} & \multicolumn{1}{|r}{0.9999} & \multicolumn{1}{|r}{0.9999} & \multicolumn{1}{|r}{0.9999} & \multicolumn{1}{|r}{-} & \multicolumn{1}{|r}{0.9} & \multicolumn{1}{|r}{0.9} & \multicolumn{1}{|r}{0.99} & \multicolumn{1}{|r}{0.99} & \multicolumn{1}{|r}{0.999} & \multicolumn{1}{|r}{-} \\
\multicolumn{1}{l}{$\gamma$} & \multicolumn{1}{|r}{-} & \multicolumn{1}{|r}{1.0} & \multicolumn{1}{|r}{2.0} & \multicolumn{1}{|r}{-} & \multicolumn{1}{|r}{-} & \multicolumn{1}{|r}{-} & \multicolumn{1}{|r}{1.0} & \multicolumn{1}{|r}{1.0} & \multicolumn{1}{|r}{-} & \multicolumn{1}{|r}{0.5} & \multicolumn{1}{|r}{0.5} & \multicolumn{1}{|r}{-} \\ \hline
\end{tabular}
\end{center}
\caption{Classification error rate of ResNet-32 trained with different loss functions on long-tailed CIFAR-10 and CIFAR-100. We show best results of class-balanced loss with best hyperparameters (SM represents Softmax and SGM represents Sigmoid) chosen via cross-validation. Class-balanced loss is able to achieve significant performance gains. $*$ denotes the case when each class has same number of samples, class-balanced term is always $1$ therefore it reduces to the original loss function.}
\label{tab:cifar_results}
\end{table*}

\subsection{Datasets}
\label{sec:data}
\textbf{Long-Tailed CIFAR.}
To analyze the proposed class-balanced loss, long-tailed versions of CIFAR~\cite{cifar} are created by reducing the number of training samples per class according to an exponential function $n = n_i \mu^i$, where $i$ is the class index ($0$-indexed), $n_i$ is the original number of training images and $\mu \in (0, 1)$. 
The test set remains unchanged.
We define the imbalance factor of a dataset as the number of training samples in the largest class divided by the smallest.
Figure~\ref{fig:cifar_data} shows number of training images per class on long-tailed CIFAR-100 with imbalance factors ranging from $10$ to $200$.
We conduct experiments on long-tailed CIFAR-10 and CIFAR-100.

\textbf{iNaturalist.}
The recently introduced iNaturalist species classification and detection dataset~\cite{inaturalist} is a real-world long-tailed dataset containing 579,184 training images from 5,089 classes in its 2017 version and 437,513 training images from 8,142 classes in its 2018 version~\cite{inat18}. We use the official training and validation splits in our experiments.

\textbf{ImageNet.}
We use the ILSVRC 2012~\cite{ilsvrc} split containing 1,281,167 training and 50,000 validation images.

Table~\ref{tab:dataset} summarizes all datasets used in our experiments along with their imbalance factors.

\subsection{Implementation}
\label{sec:implementation}

\textbf{Training with sigmoid-based losses.}
Conventional training scheme of deep networks initializes the last linear classification layer with bias $b=0$.
As pointed out by Lin~\etal~\cite{focal_loss}, this could cause instability of training when using sigmoid function to get class probabilities.
This is because using $b=0$ with sigmoid function in the last layer induces huge loss at the beginning of the training as the output probability for each class is close to $0.5$.
Therefore, for training with sigmoid cross-entropy loss and focal loss, we assume the class prior is $\pi=1/C$ for each class, where $C$ is the number of classes, and initialize the bias of the last layer as $b = -\log\left(\left(1 - \pi\right)/\pi\right)$.
In addition, we remove the $L_2$ regularization (weight decay) for the bias $b$ of the last layer.

We used Tensorflow~\cite{tensorflow} to implement and train all the models by stochastic gradient descent with momentum.
We trained residual networks with 32 layers (ResNet-32) to conduct all experiments on CIFAR.
Similar to Zagoruyko~\etal~\cite{zagoruyko2016wide}, we noticed a disturbing effect in training ResNets on CIFAR that both loss and validation error gradually went up after the learning rate drop, especially in the case of high data imbalance.
We found that setting learning rate decay to $0.01$ instead of $0.1$ solved the problem. 
Models on CIFAR were trained with batch size of $128$ on a single NVIDIA Titan X GPU for $200$ epochs.
The initial learning rate was set to $0.1$, which was then decayed by $0.01$ at $160$ epochs and again at $180$ epochs. 
We also used linear warm-up of learning rate~\cite{goyal2017accurate} in the first $5$ epochs.
On iNaturalist and ILSVRC 2012 data, we followed the same training strategy used by Goyal~\etal~\cite{goyal2017accurate} and trained residual networks with batch size of $1024$ on a single cloud TPU.
Since the scale of focal loss is smaller than softmax and sigmoid cross-entropy loss, when training with focal loss, we used $2 \times$ and $4 \times$ larger learning rate on ILSVRC 2012 and iNaturalist respectively.
Code, data and pre-trained models are available at: {\footnotesize\url{https://github.com/richardaecn/class-balanced-loss}}.

\begin{figure*}[t]
\begin{center}
\subfigure{
\includegraphics[width=1.0\columnwidth]{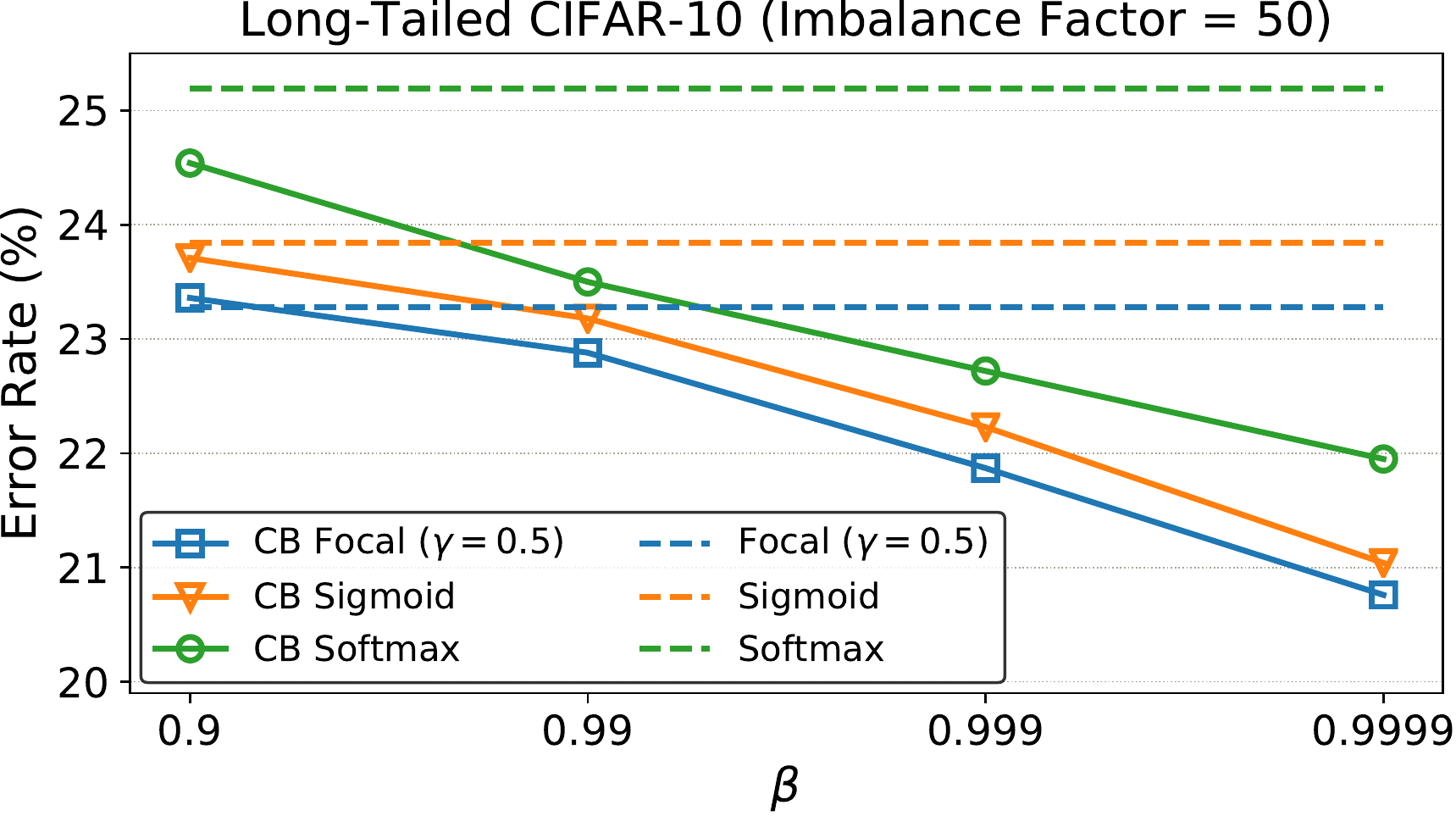}
\label{fig:cifar10_beta}}
\subfigure{
\includegraphics[width=1.0\columnwidth]{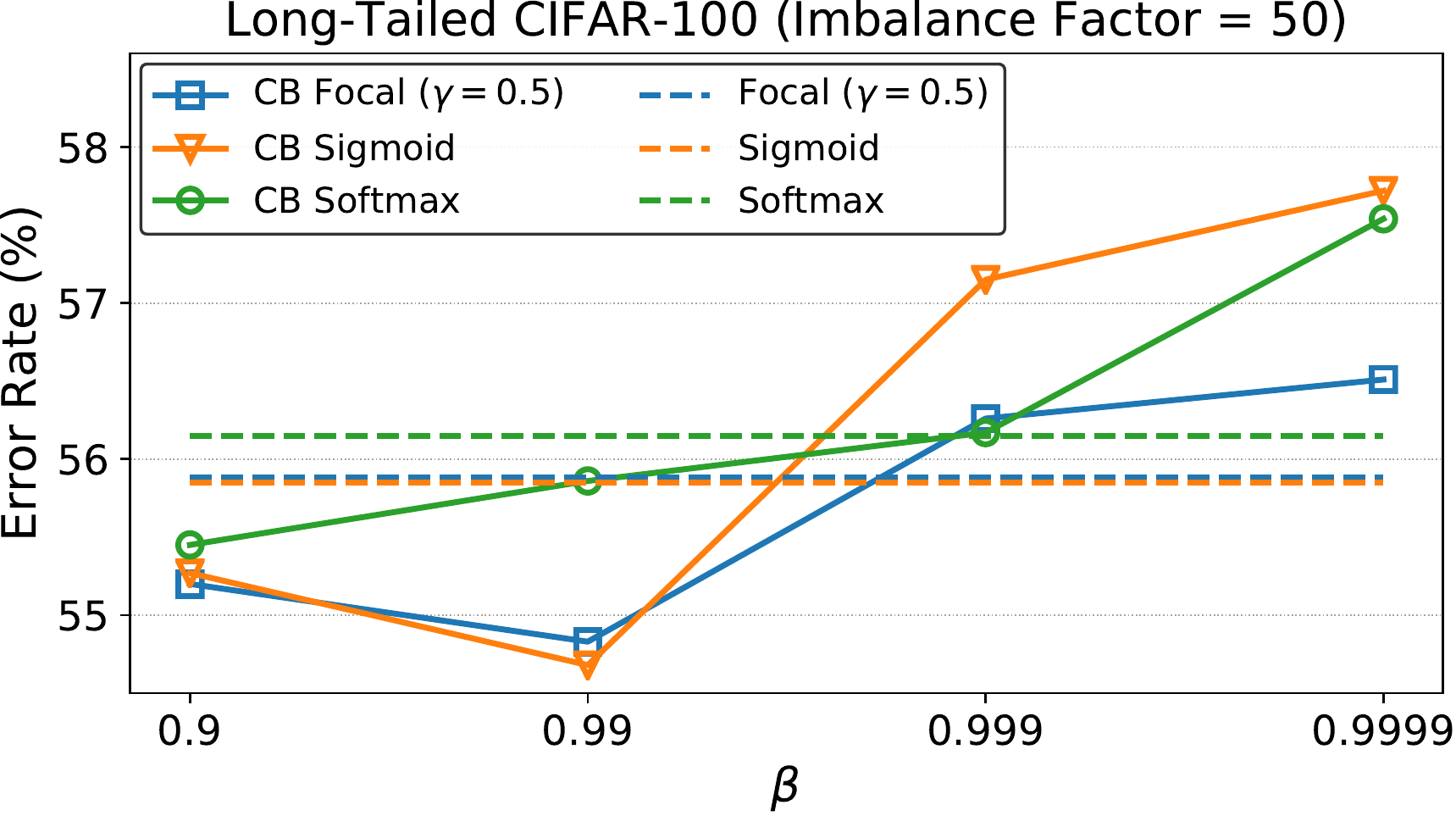}
\label{fig:cifar100_beta}}
\end{center}
\caption{Classification error rate when trained with and without the class-balanced term. On CIFAR-10, class-balanced loss yields consistent improvement across different $\beta$ and the larger the $\beta$ is, the larger the improvement is. On CIFAR-100, $\beta = 0.99$ or $\beta = 0.999$ improves the original loss, whereas a larger $\beta$ hurts the performance.}
\label{fig:cifar_beta}
\end{figure*}

\begin{figure*}[t]
\begin{center}
\subfigure{
\includegraphics[width=1.0\columnwidth]{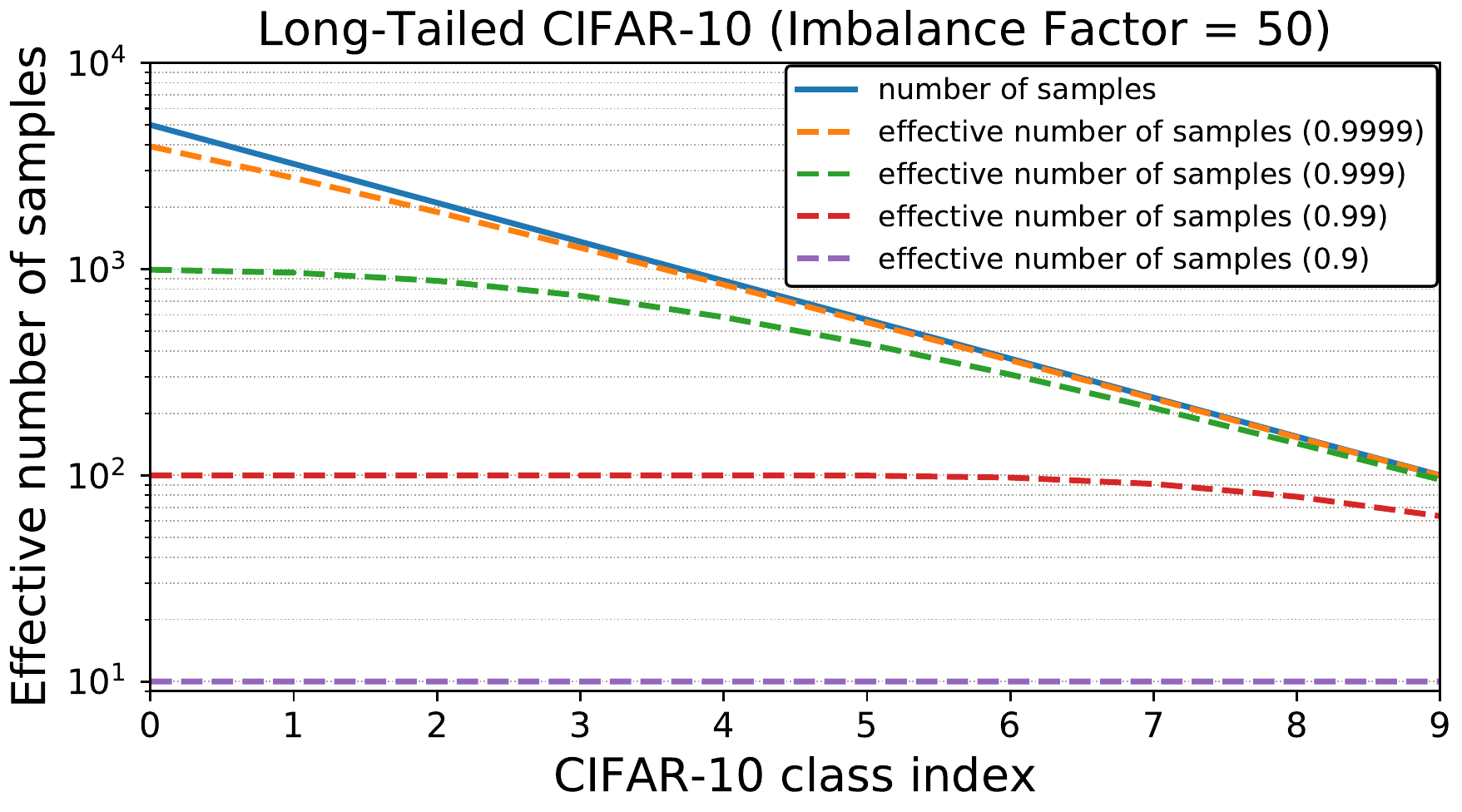}
\label{fig:cifar10_effec_num}}
\subfigure{
\includegraphics[width=1.0\columnwidth]{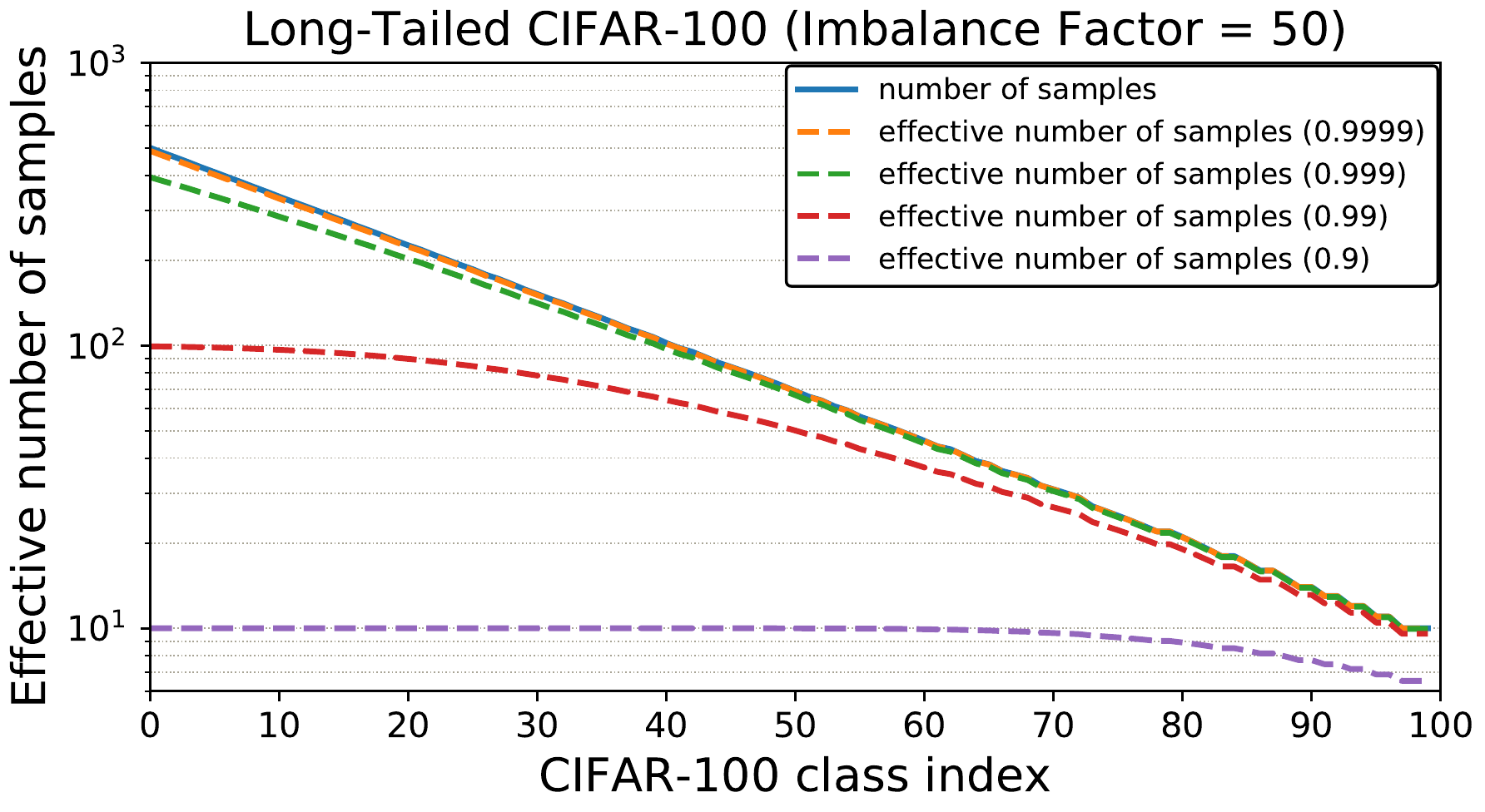}
\label{fig:cifar100_effec_num}}
\end{center}
\caption{Effective number of samples with different $\beta$ on long-tailed CIFAR-10 and CIFAR-100 with the imbalance of $50$. 
This is a semi-log plot with vertical axis in log scale.
When $\beta \to 1$, effective number of samples is same as number of samples. When $\beta$ is small, effective number of samples are similar across all classes.}
\label{fig:cifar_effec_num}
\end{figure*}

\subsection{Visual Recognition on Long-Tailed CIFAR}
\label{sec:analysis}
We conduct extensive studies on long-tailed CIFAR datasets with various imbalance factors.
Table~\ref{tab:cifar_results} shows the performance of ResNet-32 in terms of classification error rate on the test set.
We present results of using softmax cross-entropy loss, sigmoid cross-entroy loss, focal loss with different $\gamma$, and the proposed class-balanced loss with best hyperparameters chosen via cross-validation.
The search space of hyperparameters is $\{$softmax, sigmoid, focal$\}$ for loss type, $\beta \in \{0.9, 0.99, 0.999, 0.9999\}$ (Section \ref{sec:cb_loss}), and $\gamma \in \{0.5, 1.0, 2.0\}$ for focal loss~\cite{focal_loss}.

From results in Table~\ref{tab:cifar_results}, we have the following observations:
(1) With properly selected hyperparameters, class-balanced loss is able to significantly improve the performance of commonly used loss functions on long-tailed datasets.
(2) Softmax cross-entropy is overwelmingly used as the loss function for visual recognition tasks. However, following the training strategy in Section~\ref{sec:implementation}, sigmoid cross-entropy and focal loss are able to outperform softmax cross-entropy in most cases.
(3) The best $\beta$ is $0.9999$ on CIFAR-10 unanimously. But on CIFAR-100, datasets with different imbalance factors tend to have different and smaller optimal $\beta$.

To understand the role of $\beta$ and class-balanced loss better, we use the long-tailed dataset with imbalance factor of $50$ as an example to show the error rate of the model when trained with and without the class-balanced term in Figure~\ref{fig:cifar_beta}.
Interestingly, for CIFAR-10, class-balanced term always improves the performance of the original loss and more performance gain can be obtained with larger $\beta$.
However, on CIFAR-100, only small values of $\beta$ improve the performance, whereas larger values degrade the performance.
Figure~\ref{fig:cifar_effec_num} illustrates the effective number of samples under different $\beta$.
On CIFAR-10, when re-weighting based on $\beta = 0.9999$, the effective number of samples is close to the number of samples.
This means the best re-weighting strategy on CIFAR-10 is similar with re-weighting by inverse class frequency.
On CIFAR-100, the poor performance of using larger $\beta$ suggests that re-weighting by inverse class frequency is not a wise choice.
Instead, we need to use a smaller $\beta$ that has smoother weights across classes.
This is reasonable because $\beta = (N-1)/N$, so larger $\beta$ means larger $N$.
As discussed in Section~\ref{sec:effec_num}, $N$ can be interpreted as the number of unique prototypes.
A fine-grained dataset should have a smaller $N$ compared with a coarse-grained one.
For example, the number of unique prototypes of a specific bird species should be smaller than the number of unique prototypes of a generic bird class.
Since classes in CIFAR-100 are more fine-grained than CIFAR-10, CIFAR-100 should have smaller $N$ compared with CIFAR-10.
This explains our observations on the effect of $\beta$.

\begin{table*}[t]
\begin{center}
\begin{tabular}{ lcccccccccc }
\hline
&  &  & & & \multicolumn{2}{|c}{iNaturalist 2017} & \multicolumn{2}{|c}{iNaturalist 2018} & \multicolumn{2}{|c}{ILSVRC 2012} \\
\hline
\multicolumn{1}{c}{Network} & \multicolumn{1}{|c}{Loss} &
\multicolumn{1}{|c}{$\beta$} &
\multicolumn{1}{|c}{$\gamma$} &
\multicolumn{1}{|c}{Input Size} & \multicolumn{1}{|c}{Top-1} & \multicolumn{1}{|c}{Top-5} & \multicolumn{1}{|c}{Top-1} & \multicolumn{1}{|c}{Top-5} & \multicolumn{1}{|c}{Top-1} & \multicolumn{1}{|c}{Top-5} \\
\hline
\multicolumn{1}{c}{ResNet-50} & \multicolumn{1}{|c}{Softmax} &
\multicolumn{1}{|r}{-} & \multicolumn{1}{|r}{-} &
\multicolumn{1}{|r}{224 $\times$ 224} & \multicolumn{1}{|r}{45.38} & \multicolumn{1}{|r}{22.67} & \multicolumn{1}{|r}{42.86} & \multicolumn{1}{|r}{21.31} & \multicolumn{1}{|r}{23.92} & \multicolumn{1}{|r}{7.03} \\
\multicolumn{1}{c}{ResNet-101} & \multicolumn{1}{|c}{Softmax} &
\multicolumn{1}{|r}{-} & \multicolumn{1}{|r}{-} & \multicolumn{1}{|r}{224 $\times$ 224} & \multicolumn{1}{|r}{42.57} & \multicolumn{1}{|r}{20.42} & \multicolumn{1}{|r}{39.47} & \multicolumn{1}{|r}{18.86} & \multicolumn{1}{|r}{22.65} & \multicolumn{1}{|r}{6.47} \\
\multicolumn{1}{c}{ResNet-152} & \multicolumn{1}{|c}{Softmax} &
\multicolumn{1}{|r}{-} & \multicolumn{1}{|r}{-} & \multicolumn{1}{|r}{224 $\times$ 224} & \multicolumn{1}{|r}{41.42} & \multicolumn{1}{|r}{19.47} & \multicolumn{1}{|r}{38.61} & \multicolumn{1}{|r}{18.07} & \multicolumn{1}{|r}{21.68} & \multicolumn{1}{|r}{5.92} \\
\hline
\multicolumn{1}{c}{ResNet-50} & \multicolumn{1}{|c}{CB Focal} & \multicolumn{1}{|r}{0.999} & \multicolumn{1}{|r}{0.5} & \multicolumn{1}{|r}{224 $\times$ 224} & \multicolumn{1}{|r}{41.92} & \multicolumn{1}{|r}{20.92} & \multicolumn{1}{|r}{38.88} & \multicolumn{1}{|r}{18.97} & \multicolumn{1}{|r}{22.71} & \multicolumn{1}{|r}{6.72} \\
\multicolumn{1}{c}{ResNet-101} & \multicolumn{1}{|c}{CB Focal} &
\multicolumn{1}{|r}{0.999} & \multicolumn{1}{|r}{0.5} & \multicolumn{1}{|r}{224 $\times$ 224} & \multicolumn{1}{|r}{39.06} & \multicolumn{1}{|r}{18.96} & \multicolumn{1}{|r}{36.12} & \multicolumn{1}{|r}{17.18} & \multicolumn{1}{|r}{21.57} & \multicolumn{1}{|r}{5.91} \\
\multicolumn{1}{c}{ResNet-152} & \multicolumn{1}{|c}{CB Focal} &
\multicolumn{1}{|r}{0.999} & \multicolumn{1}{|r}{0.5} & \multicolumn{1}{|r}{224 $\times$ 224} & \multicolumn{1}{|r}{38.06} & \multicolumn{1}{|r}{18.42} & \multicolumn{1}{|r}{35.21} & \multicolumn{1}{|r}{16.34} & \multicolumn{1}{|r}{20.87} & \multicolumn{1}{|r}{5.61} \\
\hline
\multicolumn{1}{c}{ResNet-50} & \multicolumn{1}{|c}{CB Focal} & \multicolumn{1}{|r}{0.999} & \multicolumn{1}{|r}{0.5} & \multicolumn{1}{|r}{320 $\times$ 320} & \multicolumn{1}{|r}{38.16} & \multicolumn{1}{|r}{18.28} & \multicolumn{1}{|r}{35.84} & \multicolumn{1}{|r}{16.85} & \multicolumn{1}{|r}{21.99} & \multicolumn{1}{|r}{6.27} \\
\multicolumn{1}{c}{ResNet-101} & \multicolumn{1}{|c}{CB Focal} &
\multicolumn{1}{|r}{0.999} & \multicolumn{1}{|r}{0.5} & \multicolumn{1}{|r}{320 $\times$ 320} & \multicolumn{1}{|r}{34.96} & \multicolumn{1}{|r}{15.90} & \multicolumn{1}{|r}{32.02} & \multicolumn{1}{|r}{14.27} & \multicolumn{1}{|r}{20.25} & \multicolumn{1}{|r}{5.34} \\
\multicolumn{1}{c}{ResNet-152} & \multicolumn{1}{|c}{CB Focal} &
\multicolumn{1}{|r}{0.999} & \multicolumn{1}{|r}{0.5} & \multicolumn{1}{|r}{320 $\times$ 320} & \multicolumn{1}{|r}{33.73} & \multicolumn{1}{|r}{14.96} & \multicolumn{1}{|r}{30.95} & \multicolumn{1}{|r}{13.54} & \multicolumn{1}{|r}{19.72} & \multicolumn{1}{|r}{4.97} \\ \hline
\end{tabular}
\end{center}
\caption{Classification error rate on large-scale datasets trained with different loss functions. The proposed class-balanced term combined with focal loss (CB Focal) is able to outperform softmax cross-entropy by a large margin.}
\label{tab:large_results}
\end{table*}

\begin{figure*}[t]
\begin{center}
\subfigure{
\includegraphics[width=0.95\columnwidth]{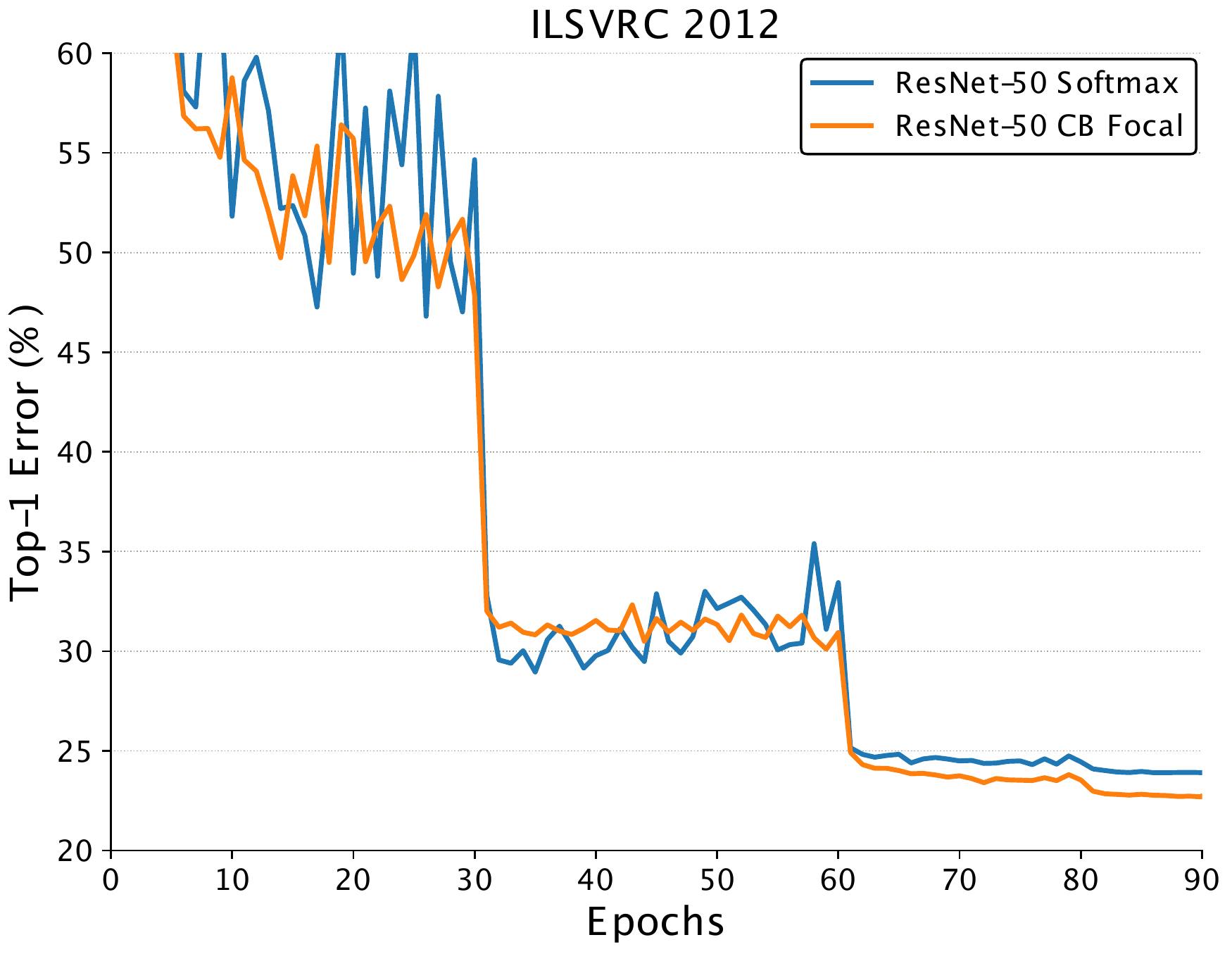}
\label{fig:curve_imagenet}}
\subfigure{
\includegraphics[width=0.95\columnwidth]{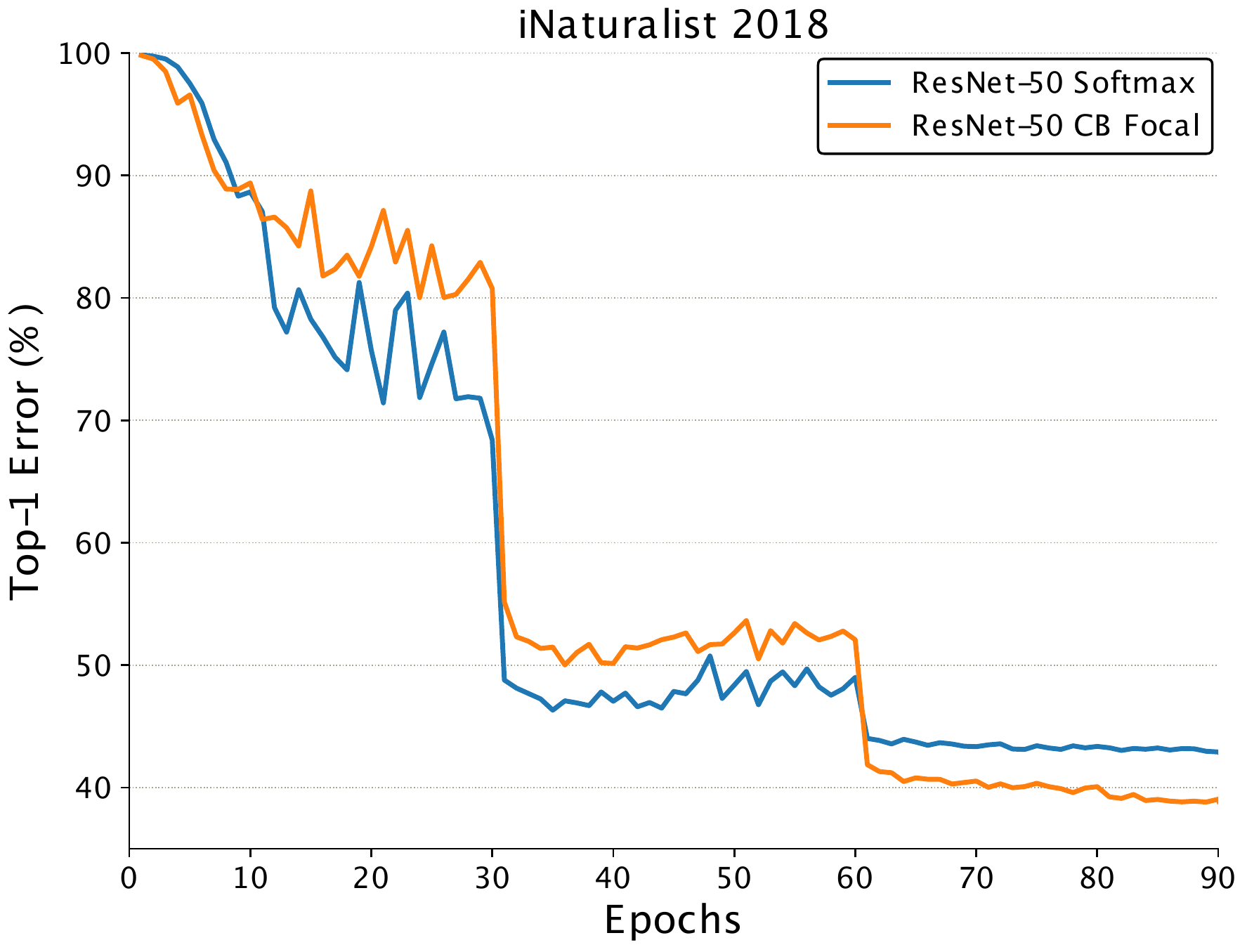}
\label{fig:curve_inat}}
\end{center}
\caption{Training curves of ResNet-50 on ILSVRC 2012 (left) and iNaturalist 2018 (right). Class-balanced focal loss with $\beta=0.999$ and $\gamma=0.5$ outperforms softmax cross-entropy after $60$ epochs.}
\label{fig:curve}
\end{figure*}

\subsection{Visual Recognition on Large-Scale Datasets}
To demonstrate the proposed class-balanced loss can be used on large-scale real-world datasets, we present results of training ResNets with different depths on iNaturalist 2017, iNaturalist 2018 and ILSVRC 2012.

Table~\ref{tab:large_results} summarizes the top-1 and top-5 error rate on the validation set of all datasets. We use the class-balanced focal loss since it has more flexibility and find $\beta = 0.999$ and $\gamma = 0.5$ yield reasonably good performance on all datasets.
From results we can see that we are able to outperform commonly used softmax cross-entropy loss on ILSVRC 2012, and by large margins on iNaturalist.
Notably, ResNet-50 is able to achieve comparable performance with ResNet-152 on iNaturalist and ResNet-101 on ILSVRC 2012 when using class-balanced focal loss to replace softmax cross-entropy loss.
Training curves on ILSVRC 2012 and iNaturalist 2018 are shown in Figure~\ref{fig:curve}.
Class-balanced focal loss starts to show its advantage after $60$ epochs of training.

\section{Conclusion and Discussion}
In this work, we have presented a theoretically sounded framework to address the problem of long-tailed distribution of training data.
The key idea is to take data overlap into consideration to help quantify the effective number of samples.
Following this framework, we further propose a~\emph{class-balanced loss} to re-weight loss inversely with the effective number of samples per class.
Extensive studies on artificially induced long-tailed CIFAR datasets have been conducted to understand and analyze the proposed loss.
The benefit of the class-balanced loss has been verified by experiments on both CIFAR and large-scale datasets including iNaturalist and ImageNet.

Our proposed framework provides a non-parametric means of quantifying data overlap, since we don't make any assumptions about the data distribution. 
This makes our loss generally applicable to a wide range of existing models and loss functions.
Intuitively, a better estimation of the effective number of samples could be obtained if we know the data distribution.
In the future, we plan to extend our framework by incorporating reasonable assumptions on the data distribution or designing learning-based, adaptive methods.

\par\noindent\textbf{Acknowledgment.}
This work was supported in part by a Google Focused Research Award.

\section*{Appendix A: More Experimental Results}
We present more comprehensive experimental results in this appendix.

\textbf{Visual Recognition on Long-Tailed CIFAR.}
On long-tailed CIFAR datasets with imbalance factors of $200$, $100$, $50$, $20$ and $10$, we trained ResNet-32 models~\cite{resnet} using softmax loss (SM), sigmoid loss (SGM) and focal loss with both the original loss and class-balanced variants with $\beta = 0.9$, $0.99$, $0.999$ and $0.9999$.
For focal loss, we used $\gamma = 0.5$, $1.0$ and $2.0$.
In addition to long-tailed CIFAR-10 and CIFAR-100 datasets mentioned in Section~\ref{sec:data} of the main paper, we also conduct experiments on CIFAR-20 dataset, which has same images as CIFAR-100 dataset but annotated with 20 coarse-grained class-labels~\cite{cifar}.

Classification error rates on long-tailed CIFAR-10, CIFAR-20 and CIFAR-100 datasets are shown in Figure~\ref{fig:cifar10_heatmap}, Figure~\ref{fig:cifar20_heatmap} and Figure~\ref{fig:cifar100_heatmap} respectively.
Each row in the figure corresponds to the model trained with a specific loss function, in the form of \{loss name\}\underline{\space}\{$\gamma$\}\underline{\space}\{$\beta$\}, and each column corresponds to a long-tailed dataset with specific imbalance factor.
We color-code each column to visualize results, with lighter colors represent lower error rates and darker colors for higher error rates.
Note that results in each column of Table~\ref{tab:cifar_results} in the main paper are classification error rates using the original losses and the best-performed class-balanced loss that is same as the lowest error rates in the corresponding column of Figure~\ref{fig:cifar10_heatmap} and Figure~\ref{fig:cifar100_heatmap} (marked by underline).
From these results, we can see that in general, higher $\beta$ yields better performance on CIFAR-10.
However, on CIFAR-20 and CIFAR-100, lower $\beta$ is needed to achieve good performance, suggesting that we cannot directly re-weight the loss by inverse class-frequency, but to re-weight based on the effective number of samples. 
These results support the analysis in Section~\ref{sec:analysis} of our main paper.

\begin{figure}[h]
\begin{center}
\includegraphics[width=1.0\columnwidth]{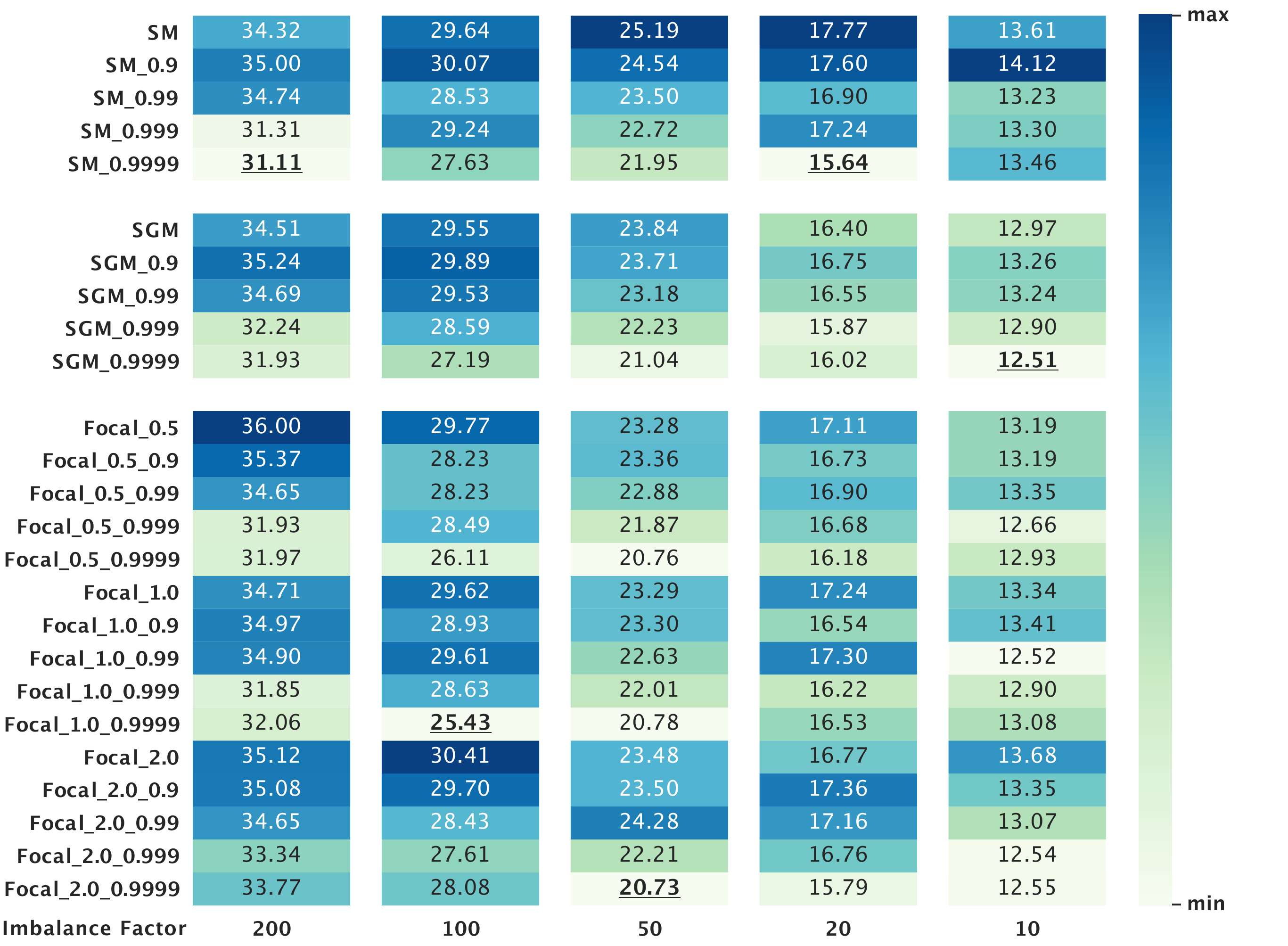}
\end{center}
\caption{Classification error rates of ResNet-32 models trained with different loss functions on long-tailed CIFAR-10.}
\label{fig:cifar10_heatmap}
\end{figure}

\begin{figure}[h]
\begin{center}
\includegraphics[width=1.0\columnwidth]{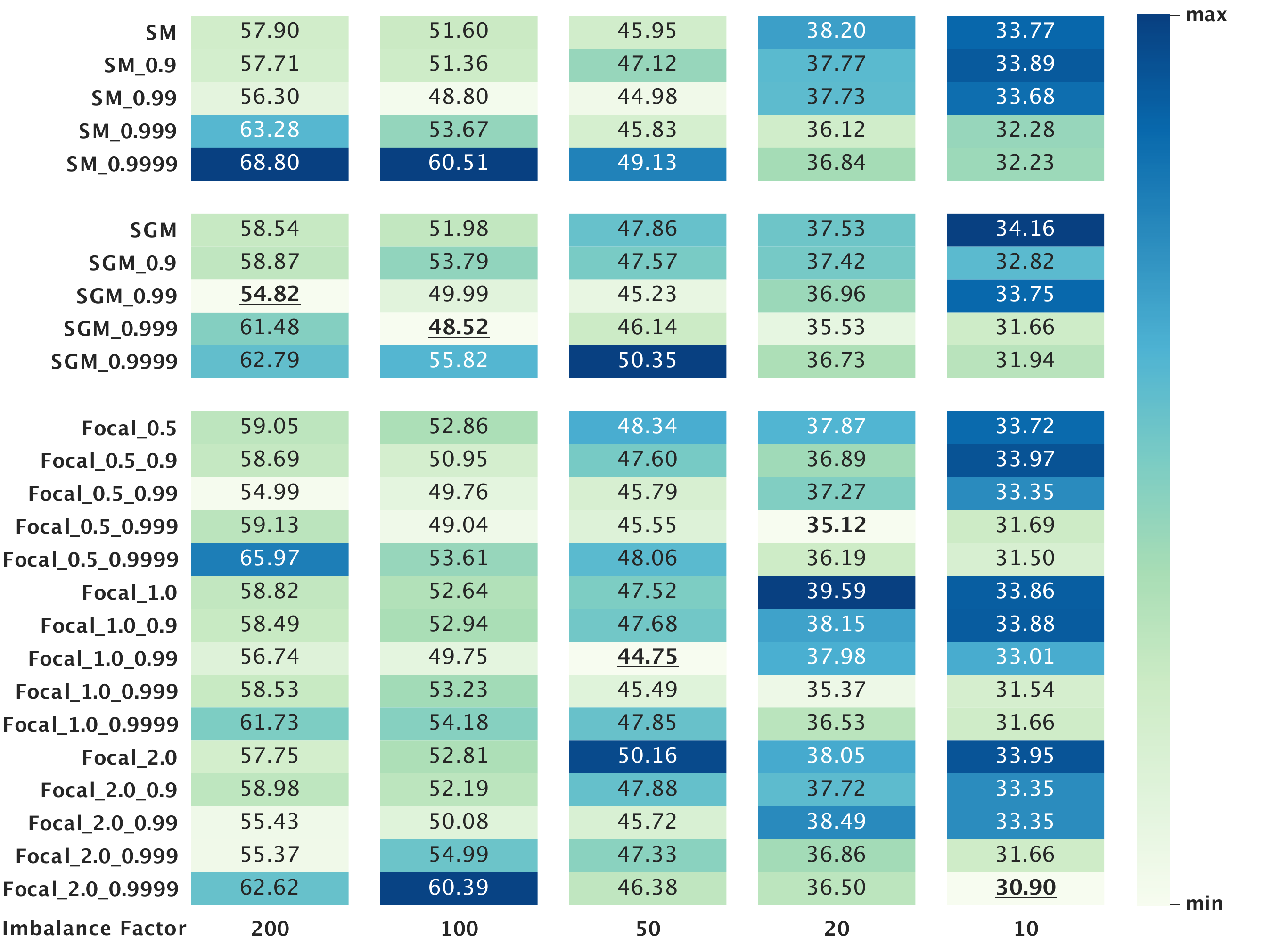}
\end{center}
\caption{Classification error rates of ResNet-32 models trained with different loss functions on long-tailed CIFAR-20 (CIFAR-100 with 20 coarse-grained class-labels).}
\label{fig:cifar20_heatmap}
\end{figure}

\begin{figure}[h]
\begin{center}
\includegraphics[width=1.0\columnwidth]{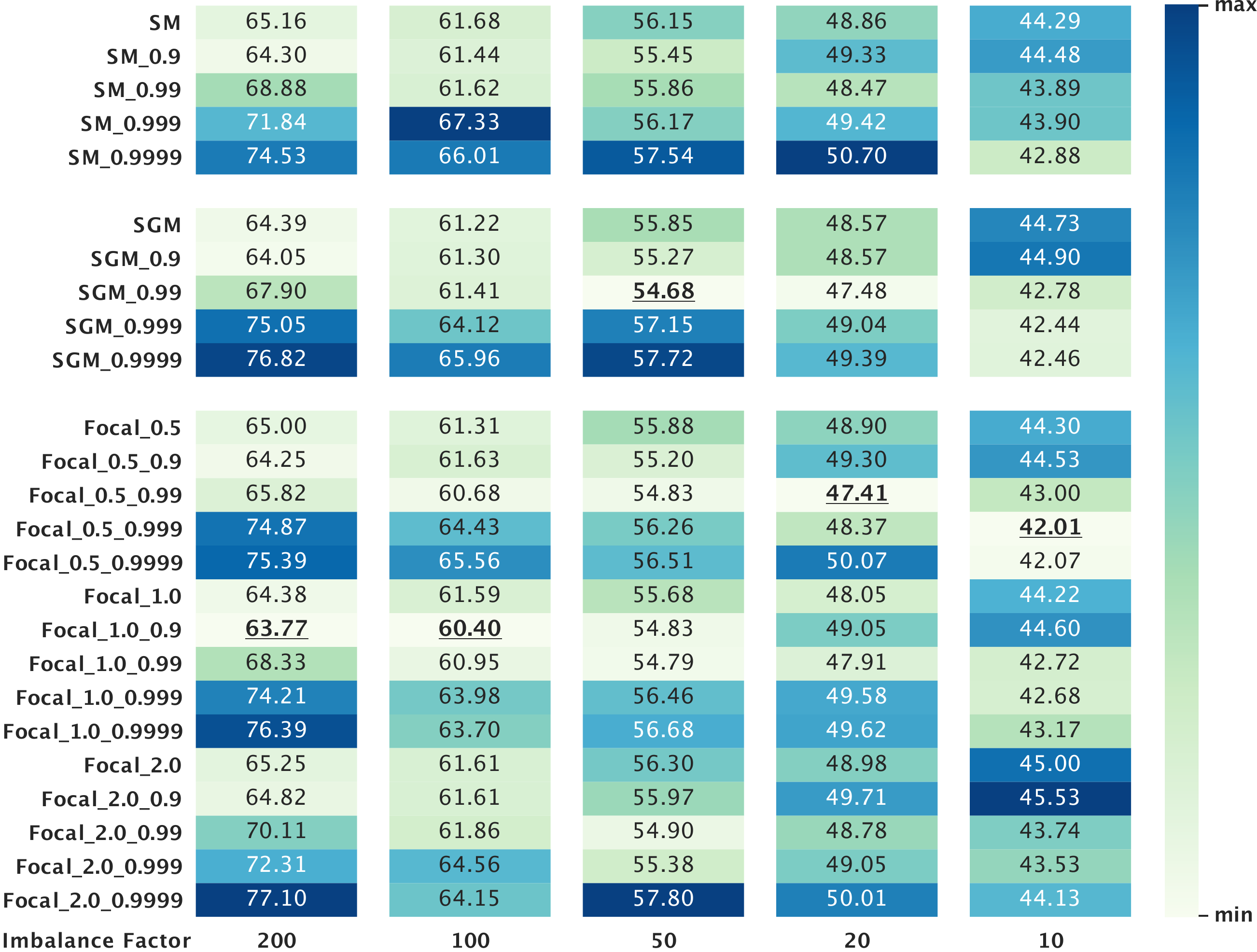}
\end{center}
\caption{Classification error rates of ResNet-32 models trained with different loss functions on long-tailed CIFAR-100.}
\label{fig:cifar100_heatmap}
\end{figure}

{\small
\bibliographystyle{ieee}
\bibliography{ref}
}

\end{document}